\documentclass{llncs}

\pdfoutput=1

\usepackage{url}
\usepackage{amsmath}
\usepackage{amsfonts}
\usepackage{mathptmx}
\DeclareMathAlphabet{\mathcal}{OMS}{cmsy}{m}{n} %
\usepackage{booktabs}
\usepackage[caption=false]{subfig}
\usepackage{relsize}

\usepackage[x11names]{xcolor}

\usepackage{tikz}
\usetikzlibrary{intersections,calc}

\DeclareMathOperator{\E}{E}     %
\DeclareMathOperator{\Var}{Var} %

\newcommand{\Rset}{\mathbb{R}}
\newcommand{\transpose}{^{\text{T}}}

\begin{document}
\pagestyle{headings}  %

\mainmatter              %

\title{Ridge Regression, Hubness, and Zero-Shot Learning\thanks{To be presented at ECML/PKDD 2015.}}

\titlerunning{Ridge Regression, Hubness, and ZSL}  %
\author{
  Yutaro Shigeto\inst{1}
  \and Ikumi Suzuki\inst{2}
  \and Kazuo Hara\inst{3}
  \and Masashi Shimbo\inst{1}
  \and\\ Yuji Matsumoto\inst{1}
}
\institute{
Nara Institute of Science and Technology, Ikoma, Nara, Japan\\
\email{\{yutaro-s,shimbo,matsu\}@is.naist.jp}\\
\and
The Institute of Statistical Mathematics, Tachikawa, Tokyo, Japan\\
\email{suzuki.ikumi@gmail.com}\\
\and
National Institute of Genetics, Mishima, Shizuoka, Japan\\
\email{kazuo.hara@gmail.com}
}

\maketitle              %

\begin{abstract}
  This paper discusses the effect of hubness in zero-shot learning, 
  when ridge regression is used to find
  a mapping between the example space to the label space.
  Contrary to the existing approach, which attempts to find
  a mapping from the example space to the label space,
  we show that
  mapping labels into the example space
  is desirable
  to suppress the emergence of hubs
  in the subsequent nearest neighbor search step.
  Assuming a simple data model,
  we prove that
  the proposed approach indeed reduces hubness.
  This was verified empirically on the tasks of
  bilingual lexicon extraction and
  image labeling:
  hubness was reduced with both of these tasks
  and the accuracy was improved accordingly.
\end{abstract}

\section{Introduction}
\label{sec:introduction}

\subsection{Background}
\label{sec:background}

In recent years, \emph{zero-shot learning} (ZSL) %
\cite{Farhadi2009,Lampert2009,Larochelle2008,Palatucci2009}
has been an active research topic in machine learning, 
computer vision, and natural language processing.
Many practical applications can be formulated as a ZSL task:
drug discovery \cite{Larochelle2008},
bilingual lexicon extraction \cite{Dinu2014,Dinu2015,Mikolov2013},
and image labeling \cite{Akata2014,Frome2013,Norouzi2014,Palatucci2009,Socher2013},
to name a few.
Cross-lingual information retrieval \cite{Vinokourov2002} can also be viewed as
a ZSL task.

ZSL can be regarded as a type of (multi-class) classification problem,
in the sense that
the classifier is given a set of known 
example-class label pairs (training set),
with the goal to predict the unknown labels of new examples
(test set).
However, ZSL differs from the standard classification
in that %
the labels for the test examples are not present
in the training set. 
In standard settings,
the classifier chooses, for each test example,
a label among those observed in the training set,
but
this is not the case in ZSL.
Moreover,
the number of class labels can be huge in ZSL;
indeed,
in bilingual lexicon extraction,
labels correspond to possible translation words,
which can range over entire vocabulary of the target language.

Obviously, such a task would be intractable without further assumptions.
Labels are thus assumed to be embedded in a metric space (\emph{label space}),
and their distance (or similarity) can be measured in this space%
\footnote{
  Throughout the paper, we assume both the example and label spaces 
  are Euclidean.
}.
Such a label space can be built
with the help of background knowledge or external resources;
in image labeling tasks,
for example,
labels correspond to annotation keywords,
which can be readily represented as vectors in a Euclidean space,
either 
by using corpus statistics in a standard way,
or by using the more recent techniques for learning word representations,
such as
the continuous bag-of-words or skip-gram models \cite{Mikolov2013a}.

After a label space is established,
one natural approach would be to
use a regression technique on the training set
to obtain a mapping function from the example space to the label space.
This function could then be used for mapping unlabeled examples into the label space,
where nearest neighbor search is carried out to find the label closest to 
the mapped example.
Finally,
this label would be output as the prediction for the example.

To find the mapping function,
some researchers use the standard
linear ridge regression \cite{Dinu2014,Dinu2015,Mikolov2013,Palatucci2009},
whereas others use neural networks \cite{Frome2013,Norouzi2014,Socher2013}.

In the machine learning community, meanwhile,
the \emph{hubness phenomenon} \cite{Radovanovic2010} is
attracting attention
as a new type of the ``curse of dimensionality.''
This phenomenon is concerned
with nearest neighbor methods in high-dimensional space,
and states that
a small number of objects in the dataset, or \emph{hubs}, may occur
as the nearest neighbor of many objects.
The emergence of these hubs will diminish the utility of nearest neighbor search,
because the list of nearest neighbors often contain
the same hub objects 
regardless of the query object for which the list is computed.

\subsection{Research Objective and Contributions}
\label{sec:research-objective}

In this paper,
we show the interaction between
the regression step in ZSL and the subsequent nearest neighbor step
has a non-negligible effect on the prediction accuracy.

In ZSL,
examples and labels are represented as vectors in high-dimensional space,
of which the dimensionality is typically a few hundred.
As demonstrated by Dinu and Baroni \cite{Dinu2015} (see also Sect.~\ref{sec:experiment}),
when ZSL is formulated
as a problem of ridge regression from examples to labels,
``hub'' labels emerge, which
are simultaneously the nearest neighbors of many mapped examples.
This has the consequence of incurring bias in the prediction,
as these labels are output as the predicted labels for these examples.
The presence of hubs are not necessarily disadvantageous in standard classification settings;
there may be ``good'' hubs as well as ``bad'' hubs \cite{Radovanovic2010}.
However, in typical ZSL tasks in which the label set is fine-grained and huge,
hubs are nearly always harmful to the prediction accuracy.

Therefore,
the objective of this study is to investigate ways
to suppress hubs, and to improve the ZSL accuracy.
Our contributions are as follows.
\begin{enumerate}  
  \item 
    We analyze the mechanism behind
    the emergence of hubs in ZSL,
    both with ridge regression
    and ordinary least squares. 
    It is established that
    hubness occurs in ZSL not only because
    of high-dimensional space,
    but also because 
    ridge regression has conventionally been used in ZSL in a way that \emph{promotes} hubness.
    To be precise,
    the distributions of the mapped examples and the labels are different 
    such that hubs are likely to emerge.
  \item 
    Drawing on the above analysis,
    we propose using ridge regression to 
    map labels into
    the space of examples.
    This approach is contrary to that followed in
    existing work on ZSL,
    in which
    examples are mapped into label space.
    Our proposal is therefore to reverse the mapping direction.

    As shown in Sect.~\ref{sec:experiment},
    our proposed approach outperformed the existing approach
    in an empirical evaluation using both synthetic and real data.

  \item 
    In terms of contributions to the research on hubness,
    this paper is the first to provide in-depth analysis of the situation
    in which 
    the query and data follow different distributions,
    and to show that
    the variance of data matters to hubness.
    In particular,
    in Sect.~\ref{sec:hubness},
    we provide a proposition in which
    the degree of bias present in the data,
    which causes hub formation,
    is expressed as a function of the data variance.
    In Sect~\ref{sec:regression-hubness},
    this proposition serves as the main tool for analyzing hubness in ZSL.

\end{enumerate}

\section{Zero-Shot Learning as a Regression Problem}
\label{sec:zsl}

Let $X$ be a set of examples,
and $Y$ be a set of class labels.
In ZSL, 
not only examples %
but also labels %
are assumed to be vectors.
For this reason,
examples are sometimes referred to as \emph{source objects},
and labels as \emph{target objects}.
In the subsequent sections of this paper,
we mostly follow this terminology %
when referring to the members of $X$ and $Y$.

Let
$X\subset \Rset^c$ and $Y\subset \Rset^d$.
These spaces, $\Rset^c$ and $\Rset^d$,
are called \emph{source space} and \emph{target space}, respectively.
Although $X$ can be the entire space $\Rset^c$,
$Y$ is usually a finite set of points in $\Rset^d$,
even though its size may be enormous in some problems.

Let 
$X_{\text{train}} = \{ \mathbf{x}_i \mid i = 1, \ldots, n \}$ be the training examples (training source objects),
and
$Y_{\text{train}} = \{ \mathbf{y}_i \mid i = 1, \ldots, n \}$ be their labels (training target objects);
i.e., the class label of example $\mathbf{x}_i$ is $\mathbf{y}_i$, for each $i = 1, \ldots, n$.
In a standard classification setting,
the labels
in the training set are equal to the entire set of labels; i.e., $Y_{\text{train}} = Y$.
In contrast, this assumption is not made in ZSL, and
$Y_{\text{train}}$ is a strict subset of $Y$.
Moreover,
it is assumed that 
the true class labels of
test examples do not belong to $Y_{\text{train}}$;
i.e., they belong to $Y \backslash Y_{\text{train}}$.

In such a situation,
it is difficult to
find a function $f$ that maps $\mathbf{x} \in X$ directly to a label in $Y$.
Therefore,
a popular (and also natural) approach 
is to learn a projection $m: \Rset^c \to \Rset^d$,
which can be done with a regression technique.
With a projection function $m$ at hand,
the label of a new source object $\mathbf{x} \in \Rset^c$ is predicted to be the one 
closest to the mapped point $m(\mathbf{x})$ in the target space.
The prediction function $f$ is thus given by
\begin{equation*}
  f(\mathbf{x}) = \mathop{\arg\min}_{\mathbf{y}\in Y} \| m(\mathbf{x}) - \mathbf{y} \|.
\end{equation*}
After a source object $\mathbf{x}$ is projected to $m(\mathbf{x})$,
the task is reduced to that of nearest neighbor search in the target space.

\section{Hubness Phenomenon and the Variance of Data}
\label{sec:hubness}

The utility of nearest neighbor search would be significantly reduced
if the same objects were to appear consistently as the search result,
irrespective of the query.
Radovanovi\'c et~al. \cite{Radovanovic2010}
showed that such objects, termed \emph{hubs}, indeed occur
in high-dimensional space.
Although this phenomenon may seem counter-intuitive, 
hubness is observed in a variety of real datasets
and distance/similarity measures used in combination \cite{Radovanovic2010,Schnitzer2012,Suzuki2013}.

The aim of this study is to
analyze the hubness phenomenon in ZSL,
which involves nearest neighbor search
in high-dimensional space as the last step.
However, as a tool for analyzing ZSL,
the existing theory on hubness \cite{Radovanovic2010}
is inadequate,
as it was mainly developed for comparing the emergence of hubness
in spaces of different dimensionalities.

In the analysis of ZSL %
in Sect.~\ref{sec:two-eggs},
we aim to compare two distributions
in the same space,
but which differ in terms of \emph{variance}.
To this end,
we first present a proposition below, 
which is similar in spirit to
the main theorem of Radovanovi\'c et~al. \cite[Theorem~1]{Radovanovic2010},
but which distinguishes the query and data distributions,
and
also expresses the expected difference between the squared distances
from queries to database objects in terms of their variance.

The proposition is concerned with nearest neighbor search,
in which
$\mathbf{x}$ is a query,
and
$\mathbf{y}_1$ and $\mathbf{y}_2$ are two objects in a dataset.
In the context of ZSL as formulated in Sect.~\ref{sec:zsl},
$\mathbf{x}$ represents the image
of a source object in the target space (through the learned regression function $m$),
and $\mathbf{y}_1$ and $\mathbf{y}_2$ are target objects (labels)
lying at different distances from the origin.
We are interested in which of $\mathbf{y}_1$ and $\mathbf{y}_2$
are more likely to be closer to $\mathbf{x}$,
when $\mathbf{x}$ is sampled from a distribution $\mathcal{X}$ with zero mean.

Let $\E[\cdot]$ and $\Var[\cdot]$ denote the expectation and variance, respectively,
and let $\mathcal{N}(\boldsymbol{\mu}, \boldsymbol{\Sigma})$ be a multivariate normal distribution
with mean $\boldsymbol{\mu}$ and covariance matrix $\boldsymbol{\Sigma}$.

\begin{proposition}
  \label{prop:radovanovic}
  Let $\mathbf{y} = [y_1, \ldots, y_d]^{\text{\upshape T}}$ be a $d$-dimensional random vector,
  with components $y_i$ ($i=1,\ldots,d$) sampled i.i.d. 
  from a normal distribution with zero mean and
  variance $s^2$;
  i.e.,
  $\mathbf{y} \sim \mathcal{Y}$, where $\mathcal{Y} = \mathcal{N}(\mathbf{0}, s^2\mathbf{I})$.
  Further let
  $\sigma = \sqrt{\Var_{\mathcal{Y}}[\|\mathbf{y}\|^2]}$
  be the standard deviation of the squared norm $\|\mathbf{y}\|^2$.

  Consider two fixed samples $\mathbf{y}_1$ and $\mathbf{y}_2$ of random vector $\mathbf{y}$,
  such that the squared norms of $\mathbf{y}_1$ and $\mathbf{y}_2$ are
  $\gamma \sigma$ apart.
  In other words,
  \begin{equation*}
    \|\mathbf{y}_2 \|^2 - \|\mathbf{y}_1 \|^2 = \gamma \sigma.
  \end{equation*}
  Let
  $\mathbf{x}$ be a point sampled from 
  a distribution $\mathcal{X}$ with zero mean.
  Then,
  the expected difference $\Delta$ between the squared distances from $\mathbf{y}_1$ and $\mathbf{y}_2$
  to $\mathbf{x}$, i.e.,
  \begin{equation}
    \Delta 
    =
    \E_{\mathcal{X}} \left[ \| \mathbf{x} - \mathbf{y}_2 \|^2 \right]
    -
    \E_{\mathcal{X}} \left[ \| \mathbf{x} - \mathbf{y}_1 \|^2 \right]
    \label{eq:delta-definition}
  \end{equation}
  is given by
  \begin{equation}
    \Delta = \sqrt{2} \gamma d^{1/2} s^2.
    \label{eq:delta}
  \end{equation}
\end{proposition}

\begin{proof}
  For $i=1,2$, the distance between a point $\mathbf{x}$ and $\mathbf{y}_i$ is given by
  \begin{equation*}
    \| \mathbf{x} - \mathbf{y}_i \|^2 
    = \| \mathbf{x} \|^2  + \| \mathbf{y}_i \|^2  - 2 \mathbf{x}^{\text{T}} \mathbf{y}_i,
  \end{equation*}
  and
  its expected value is
  \begin{align*}
    \E_{\mathcal{X}} \left[ \| \mathbf{x} - \mathbf{y}_i \|^2 \right]
    & = \E_{\mathcal{X}} \left[ \| \mathbf{x} \|^2 \right] + \| \mathbf{y}_i \|^2 - 2 \E_{\mathcal{X}} \left[ \mathbf{x} \right]^{\text{T}} \mathbf{y}_i  
      = \E_{\mathcal{X}} \left[ \| \mathbf{x} \|^2 \right] + \| \mathbf{y}_i \|^2 ,
  \end{align*}
  since $\E_{\mathcal{X}}\left[\mathbf{x}\right] = 0$ by assumption.
  Substituting this equality in \eqref{eq:delta-definition}
  yields
  \begin{align}
    \Delta 
       & = 
         \overbrace{
         \left(
         \E_{\mathcal{X}} \left[ \| \mathbf{x} \|^2 \right] + \| \mathbf{y}_2 \|^2  
         \right) 
         }^{ \E_{\mathcal{X}} \left[ \| \mathbf{x} - \mathbf{y}_2 \|^2 \right] }
         - 
         \overbrace{
         \left(
         \E_{\mathcal{X}} [ \| \mathbf{x} \|^2 ] + \| \mathbf{y}_1 \|^2
         \right)
         }^{ \E_{\mathcal{X}} \left[ \| \mathbf{x} - \mathbf{y}_1 \|^2 \right] }
         =
         \| \mathbf{y}_2 \|^2 - \| \mathbf{y}_1 \|^2
         =
         \gamma \sigma. %
         \label{eq:diff-variance}
  \end{align}

  Now, it is well known that if
  a $d$-dimensional random vector
  $\mathbf{z}$
  follows the multivariate standard normal distribution $\mathcal{N} (\mathbf{0}, \mathbf{I})$,
  then 
  its squared norm $\|\mathbf{z}\|^2$
  follows the chi-squared distribution with $d$ degrees of freedom,
  and its variance is $2d$.
  Since $\mathbf{y} = s \mathbf{z} $, %
  the variance $\sigma^2$ of the squared norm $\| \mathbf{y} \|^2$ is
  \begin{equation}
    \sigma^2
    = \Var_{\mathcal{Y}} \left[ \| \mathbf{y} \|^2 \right]
    = \Var_{\mathcal{Z}} \left[ s^2 \| \mathbf{z} \|^2 \right]
    = s^4 \Var_{\mathcal{Z}} \left[ \| \mathbf{z} \|^2 \right]
    = 2d s^4 .
    \label{eq:norm-variance}
  \end{equation}
  From \eqref{eq:diff-variance} and \eqref{eq:norm-variance},
  we obtain
  $\Delta = \gamma s^2 \sqrt{2d}$.
  \qed
\end{proof}

Note that in Proposition~\ref{prop:radovanovic},
the standard deviation $\sigma$ is used as a yardstick of measurement
to allow for comparison of ``similarly'' located object pairs across different distributions;
two object pairs in different distributions are regarded as similar
if objects in each pair are $\gamma\sigma$ apart as measured by the
$\sigma$ for the respective distributions, but has an equal factor $\gamma$.
This technique is due to
Radovanovi\'c et~al. \cite{Radovanovic2010}.

Now, $\Delta$
represents
the expected difference between the squared distances from
$\mathbf{x}$ to $\mathbf{y_1}$ and $\mathbf{y_2}$.
Equation~\eqref{eq:delta}
shows that $\Delta$ increases with $\gamma$,
the factor quantifying the amount of difference $\| \mathbf{y}_2 \|^2 - \| \mathbf{y}_1 \|^2 $.
This suggests that
a query object sampled from $\mathcal{X}$
is more likely to be closer to
object $\mathbf{y}_1$ than to $\mathbf{y}_2$,
if $\|\mathbf{y}_1\|^2 < \| \mathbf{y}_2 \|^2$;
i.e.,
$\mathbf{y}_1$ is closer to the origin than $\mathbf{y}_2$ is.
Because this holds for any pair of objects $\mathbf{y}_1$ and $\mathbf{y}_2$ in the dataset,
we can conclude that the objects closest to the origin in the dataset
tend to be hubs.

Equation~\eqref{eq:delta}
also states the relationship
between $\Delta$ and 
the component variance $s^2$ of distribution $\mathcal{Y}$,
by which the following is implied:
For a fixed query distribution $\mathcal{X}$,
if we have 
two distributions for $\mathbf{y}$, %
$\mathcal{Y}_1 = \mathcal{N}(\mathbf{0}, s_1^2\mathbf{I})$ and
$\mathcal{Y}_2 = \mathcal{N}(\mathbf{0}, s_2^2\mathbf{I})$
with $s_1^2 < s_2^2$,
it is preferable to choose $\mathcal{Y}_1$,
i.e., the distribution with a smaller $s^2$, 
when attempting to reduce hubness.
Indeed, 
assuming the independence of $\mathcal{X}$ and $\mathcal{Y}$,
we can show that
the influence of $\Delta$
relative to the expected squared distance from $\mathbf{x}$
to $\mathbf{y}$
(which is also subject to whether $\mathbf{y} \sim \mathcal{Y}_1$ or $\mathcal{Y}_2$),
is weaker for $\mathcal{Y}_1$ than for $\mathcal{Y}_2$, i.e.,
\begin{equation*}
  \frac{
    \Delta(\gamma, d, s_1)
  }{
    \E_{\mathcal{X}\mathcal{Y}_1}[\| \mathbf{x} - \mathbf{y}\|^2]
  }
  <
  \frac{
    \Delta(\gamma, d, s_2)
  }{
    \E_{\mathcal{X}\mathcal{Y}_2}[\| \mathbf{x} - \mathbf{y}\|^2]
  },
\end{equation*}
where we wrote $\Delta$ explicitly as a function of  $\gamma$, $d$, and $s$.

\section{Hubness in Regression-Based Zero-Shot Learning}
\label{sec:regression-hubness}

In this section, we analyze the emergence of hubs in 
the nearest neighbor step of ZSL.
Through the analysis,
it is shown that
hubs are promoted by the use of ridge regression
in the existing formulation of ZSL,
i.e., mapping source objects (examples) into the target (label) space.

As a solution,
we propose using ridge regression in a direction opposite to that in
existing work. That is,
we project target objects in the space of source objects,
and carry out nearest neighbor search in the source space.
Our argument for this approach consists of three steps. 

\begin{enumerate}
\item 
  We first show in Sect.~\ref{sec:shrinkage} that,
  with ridge regression (and ordinary least squares as well),
  mapped observation data tend to lie closer to the origin than the target responses do.
  Because the existing work formulates ZSL
  as a regression problem that projects source objects into the target space,
  this means that
  the norm of the projected source objects tends to be smaller than
  that of target objects.

\item 
  By combining the above result with
  the discussion of %
  Sect.~\ref{sec:hubness},
  we then argue that
  placing source objects closer to the origin
  is not ideal from the perspective of reducing hubness.
  On the contrary,
  placing target objects closer to the origin,
  as attained with the proposed approach,
  is more desirable (Sect.~\ref{sec:two-eggs}).

\item
  In Sect.~\ref{sec:nn-balls}, 
  we present a simple additional argument against
  placing source objects closer %
  to the origin;
  if the data is unimodal,
  such a configuration increases the possibility of another target object falling closer
  to the source object.
  This argument diverges from the discussion on hubness, %
  but again justifies
  the proposed approach.
\end{enumerate}

\subsection{Shrinkage of Projected Objects}
\label{sec:shrinkage}

We first prove that ridge regression
tends to map observation data closer to the origin of the space.
This tendency may be easily observed in ridge regression,
for which the penalty term shrinks 
the estimated coefficients towards zero.
However, the above tendency is also inherent in ordinary least squares.

Let $\| \cdot \|_{\text{F}}$ and $\| \cdot \|_2$
respectively
denote the Frobenius norm and the 2-norm of matrices.

\begin{proposition}
  \label{prop:shrinkage}
  Let $\mathbf{M} \in \Rset^{d \times c}$ be the solution
  for ridge regression with
  an observation matrix $\mathbf{A} \in \Rset^{c \times n}$ and
  a response matrix $\mathbf{B} \in \Rset^{d\times n}$; i.e.,
  \begin{equation}
    \label{eq:ridge-regression-objective}
    \mathbf{M} =
    \mathop{\arg\min}_{\mathbf{M}}
    \left(
      \|\mathbf{M}\mathbf{A} - \mathbf{B}\|^2_{\text{\upshape{F}}}
      +
      \lambda \| \mathbf{M} \|_{\text{\upshape{F}}}
    \right).
  \end{equation}
  where $\lambda \ge 0$ is a hyperparameter.
  Then, we have $\| \mathbf{M} \mathbf{A} \|_2 \le \| \mathbf{B} \|_2$.
\end{proposition}
\begin{proof}[Sketch]
  It is well known that
  $\mathbf{M} = \mathbf{B} \mathbf{A}\transpose \left( \mathbf{A}\mathbf{A}\transpose + \lambda \mathbf{I} \right)^{-1}\!$.
  Thus we have
  \begin{equation}
    \label{eq:2-norm-submultiplicativity}
    \| \mathbf{M}\mathbf{A} \|_2 
    = \| \mathbf{B} \mathbf{A}\transpose \left( \mathbf{A}\mathbf{A}\transpose + \lambda \mathbf{I} \right)^{-1} \!\! \mathbf{A} \|_2
    \leq \| \mathbf{B} \|_2 \; \| \mathbf{A}\transpose \left( \mathbf{A}\mathbf{A}\transpose + \lambda \mathbf{I} \right)^{-1} \!\! \mathbf{A} \|_2 .
  \end{equation}
  Let $\sigma$ be the largest singular value of $\mathbf{A}$.
  It can be shown that
  \begin{equation*}
    \| \mathbf{A}\transpose \left( \mathbf{A}\mathbf{A}\transpose + \lambda \mathbf{I} \right)^{-1} \!\! \mathbf{A} \|_2
    = \frac{ \sigma^2 }{ \sigma^2 + \lambda } \le 1.
  \end{equation*}
  Substituting this inequality in \eqref{eq:2-norm-submultiplicativity} establishes the proposition.
  \qed
\end{proof}

Recall that if the data is centered,
the matrix 2-norm can be interpreted as an indicator of the variance of data along its principal axis.
Proposition~\ref{prop:shrinkage} thus indicates that
the variance along the principal axis of the mapped observations $\mathbf{M}\mathbf{A}$
tends to be smaller than that of responses $\mathbf{B}$.

Furthermore,
this tendency
even persists
in the ordinary least squares with no penalty term (i.e., $\lambda = 0$),
since $ \| \mathbf{M}\mathbf{A} \|_2 \le  \| \mathbf{B} \|_2 $ still holds in this case;
note that 
$\mathbf{A}\transpose \left( \mathbf{A}\mathbf{A}\transpose \right)^{-1} \!\! \mathbf{A}$
is an orthogonal projection and its 2-norm is $1$, but the inequality in
\eqref{eq:2-norm-submultiplicativity} holds regardless.
This tendency therefore cannot be completely eliminated
by simply decreasing the ridge parameter $\lambda$ towards zero.

In existing work on ZSL,
$\mathbf{A}$ represents the (training) source objects $\mathbf{X} = [\mathbf{x}_1 \cdots \mathbf{x}_n] \in \Rset^{c\times n}$,
to be mapped into the space of target objects (by projection matrix $\mathbf{M}$);
and $\mathbf{B}$ is the matrix of labels for 
the training objects, i.e., $\mathbf{B} = \mathbf{Y} = [\mathbf{y}_1 \cdots \mathbf{y}_n]\in \Rset^{d\times n}$.
Although Proposition~\ref{prop:shrinkage} is thus only concerned with the training set,
it suggests
that the source objects at the time of testing, which are not in $\mathbf{X}$,
are also likely to be mapped closer to the origin of the target space
than many of the target objects in $\mathbf{Y}$.

\subsection{Influence of Shrinkage on Nearest Neighbor Search}
\label{sec:two-eggs}

We learned in Sect.~\ref{sec:shrinkage} that
ridge regression (and ordinary least squares) shrink
the mapped observation data towards the origin of the space,
relative to the response.
Thus,
in existing work on ZSL in which source objects $X$ are projected to the space of target objects $Y$,
the norm of the mapped source objects 
is likely to be smaller than that of the target objects.

The proposed approach, 
which was described in the beginning of Sect.~\ref{sec:regression-hubness},
follows the opposite direction:
target objects $Y$ are projected to the space of source objects $X$.
Thus, in this case, %
the norm of the mapped target objects 
is expected to be smaller than that of the source objects.

The question now is which of
these configurations is preferable for the subsequent nearest neighbor step,
and we provide an answer under the following %
assumptions:
(i) The source space and the target space are of equal dimensions;
(ii) the source and target objects are isotropically normally distributed and independent; and
(iii) the projected data is also isotropically normally distributed, except that
the variance has shrunk.

Let $\mathcal{D}_1 = \mathcal{N}(0, s_1^2 \mathbf{I})$ and
$\mathcal{D}_2 = \mathcal{N}(0, s_2^2 \mathbf{I})$ be
two multivariate normal distributions,
with $s_1^2 < s_2^2$. 
We compare two configurations of source object $\mathbf{x}$ and target objects $\mathbf{y}$:
(a)
the one in which
$\mathbf{x} \sim \mathcal{D}_1$ and
$\mathbf{y} \sim \mathcal{D}_2$,
and
(b)
the one in which
$\mathbf{x}' \sim \mathcal{D}_2$ and
$\mathbf{y}' \sim \mathcal{D}_1$
on the other hand; 
here,
the primes in (b) %
were added to distinguish %
variables in two configurations.

These two configurations are intended to model
situations in (a) existing work and (b) our proposal.
In configuration (a), 
$\mathbf{x}$ is shorter in expectation than $\mathbf{y}$,
and therefore
this approximates the situation that arises from existing work.
Configuration (b) represents the opposite situation,
and corresponds to our proposal
in which $\mathbf{y}$ is the projected vector
and thus is shorter in expectation than $\mathbf{x}$.

Now, we aim to verify whether the two configurations differ
in terms of the likeliness of hubs emerging,
using Proposition~\ref{prop:radovanovic}.
First, we scale the entire space of configuration (b) by $(s_1/s_2)$,
or equivalently,
we consider transformation of the variables by
$\mathbf{x}'' = (s_1/s_2) \mathbf{x}'$ and $\mathbf{y}'' = (s_1/s_2) \mathbf{y}'$.
Note that
because the two variables are scaled equally,
this change of variables preserves the nearest neighbor relations
among the samples.
See Fig.~\ref{fig:illustration-hub} for an illustration of the
relationship among
$\mathbf{x}$, $\mathbf{y}$, $\mathbf{x}'$, $\mathbf{y}'$, $\mathbf{x}''$, and $\mathbf{y}''$.

\begin{figure}[tb]
  \centering
  \includegraphics[width=\linewidth]{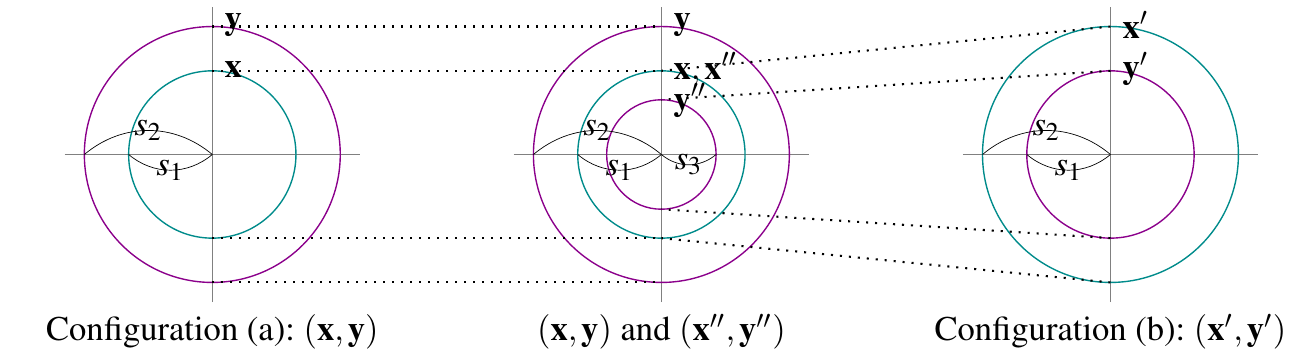}

  \caption{
    Schematic illustration  for Sect.~\ref{sec:two-eggs} in two-dimensional space.
    The left and the right panels depict configurations~(a) and (b), respectively,
    with the center panel showing both configuration (a) and the scaled version of configuration (b)
    in the same space.
    A circle represents a distribution,
    with its radius indicating the standard deviation.
    The radius of the circles for $\mathbf{x}$ (on the left panel) and $\mathbf{y}'$ (right panel) is $s_1$,
    whereas
    that of the circles for $\mathbf{y}$ (left panel) and $\mathbf{x}'$ (right panel) is $s_2$,
    with $s_1 < s_2$.
    Circles $\mathbf{x}''$ and $\mathbf{y}''$ are the scaled versions of 
    $\mathbf{x}'$ and $\mathbf{y}'$
    such that the standard deviation (radius) of $\mathbf{x}''$ is equal to $\mathbf{x}$,
    which makes the standard deviation of $\mathbf{y''}$ equal to $s_3 = s_1^2/s_2$.
  }
  \label{fig:illustration-hub}
\end{figure}

Let $\{x'_i\}$ and $\{y'_i\}$ be the components of $\mathbf{x}'$ and $\mathbf{y}'$, respectively,
and let $\{x''_i\}$ and $\{y''_i\}$ be those for  $\mathbf{x}''$ and $\mathbf{y}''$.
Then we have %
\begin{align*}
  \Var[x''_i] &
                       = \Var \left[ \frac{s_1}{s_2} x_i' \right]
                       = \left(\frac{s_1}{s_2}\right)^2 \Var[ x'_i ]
                       = s_1^2 ,
  \\
  \Var[y''_i] &
                       = \Var \left[ \frac{s_1}{s_2}y'_i \right]
                       = \left(\frac{s_1}{s_2}\right)^2 \Var [ y'_i ]
                       = \frac{s_1^4}{s_2^2} .
\end{align*}
Thus,
$\mathbf{x}''$
follows $\mathcal{N}(0, s_1^2\mathbf{I})$, and
$\mathbf{y}''$ follows $\mathcal{N}(0, (s_1^4/s_2^2)\mathbf{I})$.
Since both $\mathbf{x}$ in configuration (a) and $\mathbf{x}''$ above follow the same distribution,
it now becomes possible to compare the properties of $\mathbf{y}$ and $\mathbf{y}''$
in light of the discussion at the end of Sect.~\ref{sec:hubness}:
In order to reduce hubness,
the distribution with a smaller variance is preferred to the one with a larger variance,
for a fixed distribution of source $\mathbf{x}$ (or equivalently, $\mathbf{x}''$).

It follows that $\mathbf{y}''$ is preferable to
$\mathbf{y}$,
because the former has a smaller variance. 
As mentioned above,
the nearest neighbor relation between
the scaled variables, $\mathbf{y}''$ against $\mathbf{x}''$ (or equivalently $\mathbf{x}$),
is identical to $\mathbf{y}'$ against $\mathbf{x}'$ in configuration~(b).
Therefore, we conclude that configuration~(b) is preferable to configuration~(a),
in the sense that the former is more likely to suppress hubs.

Finally, recall that the preferred configuration~(b) models
the situation of our proposed approach,
which is to map target objects in the space of source objects.

\subsection{Additional Argument for Placing Target Objects Closer to the Origin}
\label{sec:nn-balls}

By assuming a unimodal data distribution of which
the probability density function (pdf) $p(\mathbf{z})$ is decreasing in $\|\mathbf{z}\|$,
we are able to present
the following proposition
which also advocates placing the source objects outside the target objects,
and not the other way around.

Proposition~\ref{prop:nn-balls}
is concerned with
the placement of a source object $\mathbf{x}$ at a fixed distance $r$ from its target object $\mathbf{y}$,
for which
we have two alternatives $\mathbf{x}_1$ and $\mathbf{x}_2$, located at different distances from
the origin of the space. %

\begin{proposition}
  \label{prop:nn-balls}
  Consider a finite set $Y$ of objects (i.e., points) in a Euclidean space, sampled i.i.d. from
  a distribution whose pdf $p(\mathbf{z})$ is a decreasing function of $\|\mathbf{z}\|$.
  Let $\mathbf{y} \in Y$ be an object in the set, and
  let $r > 0$. %
  Further let $\mathbf{x}_1$ and $\mathbf{x}_2$ be two objects
  at a distance $r$ apart from $\mathbf{y}$.
  If $\| \mathbf{x}_1 \| < \| \mathbf{x}_2 \|$,
  then the probability that $\mathbf{y}$ is the closest object in $Y$ to $\mathbf{x}_2$
  is greater than that of $\mathbf{x}_1$.
\end{proposition}

\begin{proof}[Sketch]
  For $i=1,2$,
  if another object in $Y$ appears within distance $r$ of $\mathbf{x}_i$,
  then $\mathbf{y}$ is not the nearest neighbor of $\mathbf{x}_i$.
  Thus, we aim to prove that
  this probability for $\mathbf{x}_2$ is smaller than that for $\mathbf{x}_1$.
  Since objects in $Y$ are sampled i.i.d, it suffices to prove
  \begin{equation}
    \label{eq:probabilities-of-balls}
    \int_{\mathbf{z} \in V_2} \!\! dp(\mathbf{z}) \,
    \le
    \int_{\mathbf{z} \in V_1} \!\! dp(\mathbf{z}),
  \end{equation}
  where
  $V_i$ ($i=1,2$) denote the balls centered at $\mathbf{x}_i$ with radius $r$.
  However,
  \eqref{eq:probabilities-of-balls} obviously holds
  because
  the balls $V_1$ and $V_2$ have the same radii,
  $p(\mathbf{z})$ is a decreasing function of $\|\mathbf{z}\|$,
  and
  $\| \mathbf{x}_1 \| \le \| \mathbf{x}_2 \|$.
  See Figure~\ref{fig:illustration} for an illustration with a bivariate standard normal distribution 
  in two-dimensional space.
  \qed  

\begin{figure}[tb]
  \centering
  \includegraphics[scale=0.95]{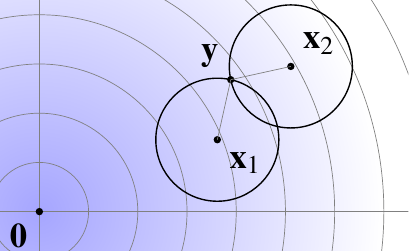}

  \caption{Illustration of the situation considered in Proposition~\ref{prop:nn-balls}.
    Here, it is assumed that
    $\|\mathbf{x}_1\| < \|\mathbf{x}_2\|$ and
    $\|\mathbf{y} - \mathbf{x}_1\| = \|\mathbf{y} - \mathbf{x}_2\|$.
    The intensity of the background shading represents the values of
    the pdf of a bivariate standard normal distribution,
    from which $\mathbf{y}$ and other objects (not depicted in the figure)
    in set $Y$ are sampled.
    The probability mass inside the circle centered at $\mathbf{x}_1$
    is greater than that centered at $\mathbf{x}_2$,
    as the intensity of the shading inside the two circles shows.}
  \label{fig:illustration}
\end{figure}

\end{proof}

In the context of existing work on ZSL,
which uses
ridge regression to map source objects in the space of target objects,
$\mathbf{x}$ can be regarded as a mapped source object,
and $\mathbf{y}$ as its target object.
Proposition~\ref{prop:nn-balls} implies that
if we want to make a source object $\mathbf{x}$ the nearest neighbor of a target object $\mathbf{y}$,
it should rather be placed farther than $\mathbf{y}$ from the origin,
but
this idea is not present in the objective function~\eqref{eq:ridge-regression-objective}
for ridge regression;
the first term of the objective allocates the same amount of penalty for $\mathbf{x}_1$ and $\mathbf{x}_2$,
as they are equally distant from the target $\mathbf{y}$.
On the contrary,
the ridge regression actually
\emph{promotes} placement of 
the mapped source object $\mathbf{x}$ closer to the origin,
as stated in Proposition~\ref{prop:shrinkage}.

\subsection{Summary of the Proposed Approach}
\label{sec:proposed-method}

Drawing on the analysis presented in Sections~\ref{sec:shrinkage}--\ref{sec:nn-balls},
we propose performing regression that maps \emph{target} objects in the
space of \emph{source} objects,
and carry out nearest neighbor search in the source space.
This opposes the approach followed in existing work on 
regression-based ZSL
\cite{Dinu2014,Dinu2015,Lazaridou2014,Mikolov2013,Palatucci2009},
which maps source objects into the space of target objects.

In the proposed approach,
matrix $\mathbf{B}$ in Proposition~\ref{prop:shrinkage}
represents the source objects $\mathbf{X}$, and
$\mathbf{A}$ represents the target objects $\mathbf{Y}$.
Therefore,
$ \| \mathbf{M}\mathbf{A} \|_2 \le  \| \mathbf{B} \|_2 $ means
$ \| \mathbf{M}\mathbf{Y} \|_2 \le  \| \mathbf{X} \|_2 $,
i.e., the mapped target objects tend to be placed closer
than
the corresponding source objects to the origin.

Admittedly,
the above argument for our proposal
relies on
strong assumptions on data distributions (such as normality),
which do not apply to real data.
However,
the effectiveness of our proposal
is verified empirically in Sect.~\ref{sec:experiment}
by using real data.

\section{Related Work}
\label{sec:related-work}

The first use of ridge regression in ZSL can be found in
the work of
Palatucci et~al. \cite{Palatucci2009}.
Ridge regression has since been one of the standard approaches to ZSL,
especially for natural language processing tasks: 
phrase generation \cite{Dinu2014} and bilingual lexicon extraction \cite{Dinu2014,Dinu2015,Mikolov2013}.
More recently,
neural networks have been used for learning
non-linear mapping %
\cite{Frome2013,Socher2013}.
All of the regression-based methods listed above,
including those based on neural networks,
map source objects 
into the target space.

ZSL can also be formulated as a problem of \emph{canonical correlation analysis} (CCA).
Hardoon et.~al. \cite{Hardoon2004} used CCA and kernelized CCA 
for image labeling.
Lazaridou et.~al. \cite{Lazaridou2014} compared 
ridge regression, CCA, singular value decomposition, and neural networks in image labeling.
In our experiments (Sect.~\ref{sec:experiment}), we use CCA as one of the baseline methods for comparison.

Dinu and Baroni \cite{Dinu2015} reported the hubness phenomenon in ZSL.
They proposed two reweighting techniques
to reduce hubness in ZSL, which are applicable to cosine similarity.
Toma\v{s}ev et~al. \cite{Tomasev2013}
proposed hubness-based instance weighting schemes for CCA.
These schemes were applied to classification problems
in which multiple instances (vectors) in the target space have the same class label.
This setting is different from the one assumed in this paper (see Sect.~\ref{sec:zsl}), i.e.,
we assume that a class label is represented by a single target vector.\footnote{
  Perhaps because of this difference,
  the method in \cite{Tomasev2013} did not perform well in our experiment,
  and we do not report its result in Sect.~\ref{sec:experiment}.}

\emph{Structured output learning} \cite{Bakir2007}
addresses a problem setting similar to ZSL,
except that the target objects typically have complex structure,
and thus the cost of embedding objects in a vector space is prohibitive.
\emph{Kernel dependency estimation} \cite{Weston2002}
is an approach that uses kernel PCA and regression to avoid this issue.
In this context,
nearest neighbor search in the target space reduces to
the \emph{pre-image} problem \cite{Mika1998} in the implicit space induced by kernels.

\section{Experiments}
\label{sec:experiment}

We evaluated the proposed approach with both synthetic and real datasets.
In particular, it was applied to two real ZSL tasks:
bilingual lexicon extraction and image labeling.

The main objective of the following experiments
is %
to verify whether our proposed approach
is capable of suppressing hub formation and outperforming
the existing approach,
as claimed in Sect.~\ref{sec:regression-hubness}.

\subsection{Experimental Setups}
\label{sec:experimental-setup}

\subsubsection{Compared Methods.}
\label{sec:compared-method}

The following methods were compared.
\begin{itemize}
 \item $\text{Ridge}_{\text{X} \to \text{Y}}$: 
 Linear ridge regression mapping source objects $X$ into the space of target objects $Y$.
   This is how ridge regression was used in the existing work on ZSL \cite{Dinu2014,Dinu2015,Lazaridou2014,Mikolov2013,Palatucci2009}.
 \item $\text{Ridge}_{\text{Y} \to \text{X}}$: 
 Linear ridge regression mapping target objects $Y$ into the source space.
   This is the proposed approach (Sect.~\ref{sec:proposed-method}).
 \item CCA: 
   Canonical correlation analysis (CCA) for ZSL \cite{Hardoon2004}.
   We used the code available from \url{http://www.davidroihardoon.com/Professional/Code.html}.
\end{itemize}

We calibrated the hyperparameters,
i.e., the regularization parameter in ridge regression and the dimensionality of common feature space in CCA,
by cross validation on the training set.

After ridge regression or CCA is applied, both X and Y (or their images) are located in the same space,
wherein we find the closest target object for a given source object as measured by the Euclidean distance. 
In addition to the Euclidean distance,
we also tested
the \emph{non-iterative contextual dissimilarity measure} (NICDM) \cite{Jegou2007}
in combination with $\text{Ridge}_{\text{X} \to \text{Y}}$ and CCA.
NICDM adjusts the Euclidean distance to make the neighborhood relations more symmetrical,
and is known to effectively reduce hubness in non-ZSL context \cite{Schnitzer2012}.

All data were centered before application of regression and CCA, as usual with these methods.

\subsubsection{Evaluation Criteria.}
\label{sec:evaluation}

The compared methods were evaluated in two respects:
(i) the correctness of their prediction,
and (ii) the degree of hubness in nearest neighbor search.

\paragraph{Measures of Prediction Correctness.}
In all our experiments, ZSL was formulated as
a ranking task;
given a source object, all the target objects were ranked by their likelihood for the source object.
As the main evaluation criterion,
we used the mean average precision (MAP) \cite{Manning2008}, 
which is one of the standard performance metrics for ranking methods.
Note that 
the synthetic and the image labeling experiments
are the single-label problems %
for which MAP is equal to the mean reciprocal rank \cite{Manning2008}.
We also report the top-$k$ accuracy\footnote{
  In image labeling (only),
  we report the top-1 accuracy ($\text{Acc}_1$) \emph{macro-averaged} over classes,
  to allow direct comparison with published results.
  Note also that $\text{Acc}_k$ with a larger $k$ would not be an informative metric
  for the image labeling task,
  which only has 10 test labels.
}
($\text{Acc}_k$) for $k=1$ and $10$,
which is
the percentage of source objects for which the correct target objects are present in their $k$ nearest neighbors.

\paragraph{Measure of Hubness.}
To measure the degree of hubness,
we used the \emph{skewness} of the (empirical) $N_{k}$ distribution,
following the approach in the literature \cite{Radovanovic2010,Schnitzer2012,Suzuki2013,Tomasev2013}.
The $N_{k}$ distribution is the distribution of
the number $N_k(i)$ of times
each target object $i$ is found in the top $k$ of the ranking
for source objects,
and its skewness is defined as follows:
\begin{equation*}
  \text{($N_{k}$ skewness)} = \frac{\mathbb \sum_{i=1}^{\ell} \left( N_k(i) - \E  \left[ N_k \right] \right)^3 / \ell }{ \Var \left[  N_k \right]^\frac{3}{2}}
\end{equation*}
where
$\ell$ is the total number of test objects in $Y$,
$N_k(i)$ is the number of times the $i$th target object
is in the top-$k$ closest target objects of the source objects.
A large $N_{k}$ skewness value indicates
the existence of target objects that frequently appear in the $k$-nearest neighbor lists of source objects;
i.e., the emergence of hubs.

\subsection{Task Descriptions and Datasets}
\label{sec:task}

We tested our method on the following ZSL tasks.

\subsubsection{Synthetic Task.}

To simulate a ZSL task,
we need to generate object pairs across two spaces
in a way that
the configuration of objects is to some extent preserved across the spaces,
but is not exactly identical.
To this end,
we first generated 3000-dimensional (column) vectors $\mathbf{z}_i \in \mathbb{R}^{3000}$ 
for $i = 1, \ldots, 10000$,
whose coordinates were generated from an i.i.d. univariate standard normal distribution.
Vectors $\mathbf{z}_i$ %
were treated
as \emph{latent} variables, in the sense that they were not directly observable,
but only their images $\mathbf{x}_i$ and $\mathbf{y}_i$ 
in two different features spaces were.
These images were obtained via different random projections, %
i.e., 
$\mathbf{x}_i = \mathbf{R}_X \mathbf{z}_i$ and
$\mathbf{y}_i = \mathbf{R}_Y \mathbf{z}_i$,
where %
$\mathbf{R}_X, \mathbf{R}_Y \in \mathbb{R}^{300 \times 3000}$
are random matrices
whose elements were sampled from the uniform distribution over $[-1,1]$.
Because random projections preserve the length and the angle of vectors in the original space
with high probability \cite{Bingham2001,Dasgupta2000},
the configuration of the projected objects is expected to be similar
(but different) across the two spaces. %

Finally,
we randomly divided object pairs $\{ (\mathbf{x}_i, \mathbf{y}_i) \}_{i=1}^{10000}$ into
the training set (8000 pairs) and the test set (remaining 2000 pairs).

\subsubsection{Bilingual Lexicon Extraction.}

Our first real ZSL task is bilingual lexicon extraction \cite{Dinu2014,Dinu2015,Mikolov2013},
formulated as a ranking task: 
Given a word in the source language, the goal is 
to rank its gold translation (the one listed in an existing bilingual lexicon 
as the translation of the source word) higher than other non-translation candidate words.

In this experiment,
we evaluated the performance in the tasks of finding the English translations of words in the
following source languages:
Czech (cs), German (de), French (fr), Russian (ru), Japanese (ja), and Hindi (hi).
Thus, in our setting, 
each of these six languages was used as $X$ alternately,
whereas English was the target language $Y$ throughout.\footnote{
  We also conducted experiments with
  English as $X$ and other languages as $Y$.
  The results are not presented here due to lack of space,
  but the same trend was observed.
}

Following related work \cite{Dinu2014,Dinu2015,Mikolov2013},
we trained
a 
CBOW %
model \cite{Mikolov2013a}
on the pre-processed Wikipedia corpus distributed by
the Polyglot project\footnote{\url{https://sites.google.com/site/rmyeid/projects/polyglot}}
(see \cite{Al-Rfou2013} for corpus statistics),
using the word2vec\footnote{\url{https://code.google.com/p/word2vec/}} tool.
The window size parameter of word2vec was set to 10, with the dimensionality of feature vectors set to 500.

To learn the projection function and measure the accuracy in the test set,
we used the bilingual dictionaries\footnote{\url{http://hlt.sztaki.hu/resources/dict/bylangpair/wiktionary_2013july/}} 
of \'Acs~et~al. \cite{Acs2013} as the gold translation pairs.
These gold pairs were randomly split into the training set (80\% of the whole pairs) and the test set (20\%).
We repeated experiments on four different random splits, for which we report the average performance.

\subsubsection{Image Labeling.}

The second real task is image labeling,
i.e., the task of finding a suitable word label for a given image.
Thus, source objects $X$ are the images and target objects $Y$ are the word labels.

We used the Animal with Attributes (AwA) dataset\footnote{
  \url{http://attributes.kyb.tuebingen.mpg.de/}
}, which consists of 30,475 images of 50 animal classes.
For image representation, 
we used the DeCAF features \cite{Donahue2013},
which are the 4096-dimensional vectors constructed with convolutional neural networks (CNNs).
DeCAF is also available from the AwA website.
To save computational cost,
we used random projection to reduce the dimensionality of DeCAF features to 500.

As with the bilingual lexicon extraction experiment,
label features (word representations) 
were constructed with word2vec,
but this time they were trained on the
English version of Wikipedia (as of March 4, 2015)
to cover all AwA labels.
Except for the corpus, we used the same word2vec parameters as with bilingual lexicon extraction.

We respected the standard zero-shot setup on AwA provided with the dataset;
i.e., the training set contained 40 labels, and test set contained the other 10 labels.

\subsection{Experimental Results}
\label{sec:results}

\captionsetup[subfloat]{position=top}

\begin{table*}[htb!]
  \smaller
  \centering
  \caption{
    Experimental results: MAP is the mean average precision. 
    $\text{Acc}_\text{k}$ is the accuracy of the $k$-nearest neighbor list.
    $N_k$ is the skewness of the $N_k$ distribution. 
    A high $N_k$ skewness indicates the emergence of hubs (smaller is better).
    The bold figure indicates the best performer in each evaluation criteria. 
  }
  \label{tab:result}

  \subfloat[
    Synthetic data.
  ]{
  \label{tab:synthetic-result}
  \begin{tabular}{l rrr rr}
    \toprule
    method                                                & \multicolumn{1}{c}{MAP}           & \multicolumn{1}{c}{$\text{Acc}_\text{1}$} & \multicolumn{1}{c}{$\text{Acc}_\text{10}$} & \multicolumn{1}{c}{$N_{1}$}       & \multicolumn{1}{c}{$N_{10}$}      \\
    \midrule
        $\text{Ridge}_{\text{X} \to \text{Y}}$            & 21.5          & 13.8                  & 36.3                   & 24.19         & 12.75         \\
        $\text{Ridge}_{\text{X} \to \text{Y}}$ + NICDM    & 58.2          & 47.6                  & 78.4                   & 13.71         & 7.94          \\
        $\text{Ridge}_{\text{Y} \to \text{X}}$ (proposed) & \textbf{91.7} & \textbf{87.6}         & \textbf{98.3}          & \textbf{0.46} & \textbf{1.18} \\
        CCA                                               & 78.9          & 71.6                  & 91.7                   & 12.0          & 7.56          \\
        CCA + NICDM                                       & 87.6          & 82.3                  & 96.5                   & 0.96          & 2.58          \\
    \bottomrule
  \end{tabular}
  }\hfill
  \subfloat[
    MAP on bilingual lexicon extraction.
  ]{
  \label{tab:ble-map}

  \begin{tabular}{l rrrrrr}
    \toprule
    method                                              & \multicolumn{1}{c}{cs}        & \multicolumn{1}{c}{de}        & \multicolumn{1}{c}{fr}         & \multicolumn{1}{c}{ru}            & \multicolumn{1}{c}{ja}            & \multicolumn{1}{c}{hi}            \\
    \midrule
    $\text{Ridge}_{\text{X} \to \text{Y}} $             & 1.7           & 1.0           & 0.7           & 0.5           & 0.9           & 5.3           \\ 
    $\text{Ridge}_{\text{X} \to \text{Y}}$ + NICDM      & 11.3          & 7.1           & 5.9           & 3.8           & 10.2          & 21.4          \\
    $\text{Ridge}_{\text{Y} \to \text{X}} $  (proposed) & \textbf{40.8} & \textbf{30.3} & \textbf{46.5} & \textbf{31.1} & \textbf{42.0} & \textbf{40.6} \\
    CCA                                                 & 24.0          & 18.1          & 33.7          & 21.2          & 27.3          & 11.8          \\
    CCA + NICDM                                         & 30.1          & 23.4          & 39.7          & 26.7          & 35.3          & 19.3          \\
    \bottomrule
  \end{tabular}
  }

  \subfloat[
    $\text{Acc}_k$ on bilingual lexicon extraction.
  ]{
  \label{tab:ble-acc}

  \begin{tabular}{l *{6}{@{\hspace{1.6em}} rr}}
    \toprule
                                                      & \multicolumn2c{cs}    & \multicolumn2c{de}     & \multicolumn2c{fr}    & \multicolumn2c{ru}
                                                      & \multicolumn2c{ja}    & \multicolumn2c{hi}                                                                                                                       \\
    \cmidrule(r{1em}){2-3} \cmidrule(r{1em}){4-5} \cmidrule(r{1em}){6-7} \cmidrule(r{1em}){8-9} \cmidrule(r{1em}){10-11} \cmidrule{12-13}
    method                                            & \multicolumn1c{$\text{Acc}_\text{1}$} & \multicolumn1c{$\text{Acc}_\text{10}$} & \multicolumn1c{$\text{Acc}_\text{1}$} & \multicolumn1c{$\text{Acc}_\text{10}$}   
    & \multicolumn1c{$\text{Acc}_\text{1}$} & \multicolumn1c{$\text{Acc}_\text{10}$} & \multicolumn1c{$\text{Acc}_\text{1}$} & \multicolumn1c{$\text{Acc}_\text{10}$}
    & \multicolumn1c{$\text{Acc}_\text{1}$} & \multicolumn1c{$\text{Acc}_\text{10}$} & \multicolumn1c{$\text{Acc}_\text{1}$} & \multicolumn1c{$\text{Acc}_\text{10}$}                                                                  \\
    \midrule
    $\text{Ridge}_{\text{X} \to \text{Y}}$            & 0.7                   & 2.8                    & 0.4                   & 1.6           & 0.3           & 1.2           & 0.2  & 0.8  & 0.2  & 1.3  & 2.9  & 8.2  \\
    $\text{Ridge}_{\text{X} \to \text{Y}}$ + NICDM    & 7.2                   & 17.9                   & 4.3                   & 11.4          & 3.5           & 9.8           & 2.1  & 6.3  & 6.1  & 16.8 & 14.4 & 32.6 \\
    $\text{Ridge}_{\text{Y} \to \text{X}}$ (proposed) & \textbf{31.5}         & \textbf{54.5}          & \textbf{21.6}         & \textbf{43.0} & \textbf{36.6} 
                                                      & \textbf{58.6}         & \textbf{21.9}          & \textbf{43.6}         & \textbf{31.9} & \textbf{56.3} & \textbf{31.1} & \textbf{55.4}                           \\
    CCA                                               & 17.9                  & 32.7                   & 12.9                  & 25.2          & 27.0          & 41.7          & 15.2 & 28.8 & 20.2 & 37.3 & 7.4  & 18.9 \\
    CCA + NICDM                                       & 21.9                  & 42.3                   & 16.1                  & 33.9          & 31.1          & 50.1          & 18.7 & 37.0 & 25.9 & 48.8 & 12.4 & 30.7 \\
    \bottomrule
  \end{tabular}
  }

  \subfloat[
    $N_k$ skewness on bilingual lexicon extraction. 
  ]{
  \label{tab:ble-skew}

  \begin{tabular}{l *{6}{@{\hspace{1.25em}} rr}}
    \toprule
                                                       & \multicolumn2c{cs} & \multicolumn2c{de} & \multicolumn2c{fr} & \multicolumn2c{ru}
                                                       & \multicolumn2c{ja} & \multicolumn2c{hi}                                                                                                                          \\
    \cmidrule(r{1em}){2-3} \cmidrule(r{1em}){4-5} \cmidrule(r{1em}){6-7} \cmidrule(r{1em}){8-9} \cmidrule(r{1em}){10-11} \cmidrule{12-13}
  method                                             & \multicolumn1c{$N_{1}$}            & \multicolumn1c{$N_{10}$}           & \multicolumn1c{$N_{1}$}          & \multicolumn1c{$N_{10}$}       & \multicolumn1c{$N_{1}$}       & \multicolumn1c{$N_{10}$}       & \multicolumn1c{$N_{1}$} & \multicolumn1c{$N_{10}$} 
    & \multicolumn1c{$N_{1}$}            & \multicolumn1c{$N_{10}$}           & \multicolumn1c{$N_{1}$}            & \multicolumn1c{$N_{10}$} \\
    \midrule
     $\text{Ridge}_{\text{X} \to \text{Y}}$            & 50.29              & 23.84              & 43.00              & 24.37          & 67.79         & 35.83          & 95.05   & 35.36 & 62.12 & 22.78 & 23.75 & 10.84 \\
     $\text{Ridge}_{\text{X} \to \text{Y}}$ + NICDM    & 41.56              & 20.38              & 39.32              & 20.82          & 57.18         & 25.97          & 89.08   & 30.70 & 57.57 & 21.62 & 20.33 & 9.21  \\
     $\text{Ridge}_{\text{Y} \to \text{X}}$ (proposed) & \textbf{11.91}     & \textbf{10.74}     & \textbf{12.49}     & \textbf{11.94} & \textbf{2.56} 
                                                       & \textbf{2.77}      & \textbf{4.28}      & \textbf{4.18}      & \textbf{5.15}  & \textbf{6.76} & \textbf{10.45} & \textbf{6.14}                                   \\
     CCA                                               & 28.00              & 18.67              & 36.66              & 18.98          & 30.18         & 15.95          & 51.92   & 21.60 & 37.73 & 18.27 & 22.31 & 8.95  \\
     CCA + NICDM                                       & 25.00              & 17.13              & 32.94              & 17.65          & 25.20         & 14.65          & 42.61   & 20.72 & 34.66 & 13.16 & 22.00 & 8.46  \\
    \bottomrule
  \end{tabular}
  }

  \subfloat[
    Image labeling.
  ]{
  \label{tab:image-result}

  \begin{tabular}{l r r r}
    \toprule
    method                                            & \multicolumn{1}{c}{MAP}        & \multicolumn{1}{c}{$\text{Acc}_1$} & \multicolumn{1}{c}{$N_{1}$}   \\
    \midrule
    $\text{Ridge}_{\text{X} \to \text{Y}}$            & 46.0          & 22.6           & 2.61          \\
    $\text{Ridge}_{\text{X} \to \text{Y}}$ + NICDM    & 54.2          & 34.5           & 2.17          \\
    $\text{Ridge}_{\text{Y} \to \text{X}}$ (proposed) & \textbf{62.5} & \textbf{41.3}  & \textbf{0.08} \\
    CCA                                               & 26.1          & 9.2            & 2.00          \\
    CCA + NICDM                                       & 26.9          & 9.3            & 2.42          \\ 
    \bottomrule
  \end{tabular}
  }

\end{table*}

Table~\ref{tab:result} shows the experimental results.
The trends are fairly clear:
The proposed approach, $\text{Ridge}_{\text{Y} \to \text{X}}$, %
outperformed
other methods in both
MAP %
and $\text{Acc}_k$,
over all tasks.
$\text{Ridge}_{\text{X} \to \text{Y}}$ %
and CCA
combined with NICDM performed better than those with Euclidean distances,
although they still lagged behind the proposed method
$\text{Ridge}_{\text{Y} \to \text{X}}$ even with NICDM.

The $N_k$ skewness achieved by $\text{Ridge}_{\text{Y} \to \text{X}}$ was lower (i.e., better)
than that of compared methods, meaning that it effectively suppressed 
the emergence of hub labels.
In contrast,
$\text{Ridge}_{\text{X} \to \text{Y}}$ produced a high skewness 
which was in line with its poor prediction accuracy.
These results support the expectation we expressed in the discussion in Sect.~\ref{sec:regression-hubness}.

The results presented in the tables show that
the degree of hubness ($N_{k}$) for all tested methods
inversely correlates with the correctness of the output rankings,
which strongly suggests that hubness is one major factor affecting the prediction accuracy.

For the AwA image dataset,
Akata et.~al. \cite[the fourth row (CNN) and second column ($\varphi^{w}$) of Table~2]{Akata2014}
reported a 39.7\% $\text{Acc}_1$ score, %
using image representations trained with CNNs,
and 100-dimensional word representations trained with word2vec.
For comparison,
our proposed approach, $\text{Ridge}_{\text{Y} \to \text{X}}$,
was evaluated in a similar setting:
We used the DeCAF features (which were also trained with CNNs) 
without random projection
as the image representation, 
and 100-dimensional word2vec word vectors.
In this setup, $\text{Ridge}_{\text{Y} \to \text{X}}$ achieved a 40.0\% $\text{Acc}_1$ score.
Although the experimental setups are not exactly identical and thus the results are not directly comparable,
this suggests that even linear ridge regression can potentially perform as well as
more recent methods, such as Akata et~al.'s,
simply by exchanging the observation and response variables. %

\section{Conclusion}
\label{sec:conclusion}

This paper has presented our formulation of ZSL as a regression problem
of finding a mapping from the target space to the source space,
which opposes the way in which regression has been applied to ZSL to date.
Assuming a simple model in which data follows a multivariate normal distribution,
we provided an explanation as to why the proposed direction is preferable,
in terms of the emergence of hubs in the subsequent nearest neighbor search step.
The experimental results showed that the proposed approach outperforms
the existing regression-based and CCA-based approaches to ZSL.

Future research topics include:
(i) extending the analysis of Sect.~\ref{sec:regression-hubness} 
to cover multi-modal data distributions,
or
other similarity/distance measures such as cosine;
(ii)
investigating the influence of mapping directions in other regression-based ZSL methods,
including neural networks; and
(iii)
investigating the emergence of hubs in CCA.

\subsubsection*{Acknowledgments.}

We thank anonymous reviewers for their valuable comments and suggestions.
MS was supported by JSPS Kakenhi Grant no.~15H02749.

\bibliographystyle{splncs03}

\end{document}